\documentclass{opt2024} 

\usepackage{thmtools}
\usepackage{thm-restate}

\usepackage{multicol}

\def\x{{\mathbf{x}}}

\def\y{{\mathbf{y}}}

\def\g{{\mathbf{g}}}

\def\A{{\mathbf{A}}}
\def\w{{\mathbf{w}}}
\def\q{{\mathbf{q}}}
\def\y{{\mathbf{y}}}
\def\I{{\mathbf{I}}}
\def\H{{\mathbf{H}}}

\def\wmin{{w_{\min}}}
\def\diag{\mathrm{diag}}
\newcommand{\reals}{\mathbb{R}}
\newcommand{\mX}{\mathcal{X}}
\newcommand{\brac}[1]{\left(#1\right)}
\newcommand{\enorm}[1]{\left\Vert#1\right\Vert}
\usepackage{bbm}
\newcommand{\inner}[2]{\langle #1, #2 \rangle}

\title[Aligned Multi Objective Optimization]{Aligned Multi Objective Optimization}

\optauthor{%
\Name{Yonathan Efroni} \addr{Meta AI}\\
\Name{Daniel Jiang} \addr{Meta AI}\\
\Name{Ben Kretzu} \addr{Technion}\\
\Name{Jalaj Bhandari} \addr{Meta AI}\\
\Name{Zheqing (Bill) Zhu} \addr{Meta AI}\\
\Name{Karen Ullrich} \addr{Meta AI}
}

\begin{document}

\maketitle

\begin{abstract}%
To date, the multi-objective optimization literature has mainly focused on conflicting objectives, studying the Pareto front or requiring users to balance tradeoffs. Yet, in machine learning practice, there are many scenarios where such conflict does not take place. Recent findings from multi-task learning, reinforcement learning, and LLMs show that diverse related tasks can enhance performance across objectives simultaneously. Despitechcuh this evidence, such phenomenon has not been examined from an optimization perspective. This leads to a lack of generic gradient-based methods that can scale to scenarios with a large number of related objectives.  To address this gap, we introduce the Aligned Multi-Objective Optimization framework, propose the \texttt{AMOOO} algorithm, and provide theoretical guarantees of its superior performance compared to naive approaches.
\end{abstract}

\section{Introduction}
In many real-world optimization problems, we have access to multi-dimensional feedback rather than a single scalar objective. The multi-objective optimization (MOO) literature has largely focused on the setting where these objectives \emph{conflict} with each other, which necessitates the \emph{Pareto dominance} notion of optimality. A closely related area of study is \emph{multi-task learning} \citep{sener2018multi}, where multiple tasks are learned jointly, typically with both shared and task-specific parameters. The hope is that the model can perform better on individual task by sharing common information across tasks. Indeed, the phenomenon of improved performance across all tasks has been observed in several settings \citep{lin2023libmtl,lee2024parrot}, suggesting that perhaps there may not be significant trade-offs between objectives.

In this paper, we explicitly consider a setting where objectives are \emph{aligned}, i.e., objectives that share a common solution. For example, in reinforcement learning, practitioners can sometimes speed up learning by exploit several alternative reward specifications that all lead to the same optimal policy \citep{dann2023reinforcement}. In statistics and machine learning, labeled data is sometimes sparse, leading practitioners to rely on closely-related proxy tasks to improve prediction accuracy \citep{bastani2021predicting}. 

To our knowledge, there is no work that studies this setting from a purely optimization perspective. We ask the question: \emph{how can an optimization algorithm benefit from multi-objective feedback when the objectives are aligned?}  
We introduce the \emph{aligned multi-objective optimization} (AMOO) framework to study this question. Subsequently, we design a new algorithm with provable guarantees for the AMOO setting and show empirical evidence of improved convergence properties.

    



\section{Aligned Multi Objective Optimization}\label{sec:AMOO setting}

Consider an unconstrained multi-objective optimization where
 $F: \reals^{n} \to \reals^m$ is a vector valued function, 
 $
 F(\x) = \brac{f_1(\x), f_2(\x), \ldots, f_m(\x) },
 $
 and all functions $\{ f_i\}_{i\in [m]}$ are convex where $[m]:=\{1,\ldots,m\}$.
Without additional assumptions the components of $F(\x)$ cannot be minimized simultaneously. To define a meaningful approach to optimize $F(\x)$ one can study the Pareto front, or to properly define how to trade-off the objectives. We denote by $\Delta_m$ the $m$-dimensional simplex, and by $\Delta_{m,\alpha} := \{ \w\in \mathbb{R}^m: \w\in \Delta_m, \forall i\in [m]\ w_i\geq \alpha\}$. In the AMOO setting we make the assumption the functions are aligned in a specific sense: we assume that the functions $\{ f_i\}_{i\in [m]}$ share an optimal solution. Namely, there exists a point $\x^*$ that minimizes all functions in $F(\cdot)$ simultaneously,
\begin{align}
    \x^*\in \arg\min_{x\in \reals^n} f_i(\x) \quad \forall i \in [m]. \label{eq:aligned_functions}
\end{align}
With this additional assumption one may hope to get quantitative benefits from the multi objective feedback. 
How can Gradient Descent (GD) be improved when the functions are aligned? 


A common algorithmic approach in the multi-objective setting is using a weight vector $\w\in \mathbb{R}^m$ that  maps the vector $F(\x)$ into a single objective $f_{\w}(\x) := \w^T F(\x)$, amenable to GD optimization~\citep{sener2018multi, yu2020gradient, liu2021conflict,navon2022multi}. Existing algorithms suggest alternatives for choosing $\w$. We follow this paradigm and design an algorithm that chooses the weights adaptively for the AMOO setting.

Towards developing intuition for our algorithmic approach we consider few examples of the AMOO setting. These showcase the need to choose weights in an adaptive way to the problem.

\paragraph{The Specification Example.}

Consider the case $F(\x)=(f_1(\x),f_2(\x))$, $\x\in \reals^2$ where
\begin{align*}
    &f_1(\x) = (1-\Delta) x_1^2 + \Delta x_2^2, \quad \text{and}\quad f_2(\x) = \Delta x_1^2 +  (1-\Delta) x_2^2,
\end{align*}
for some small $\Delta\in [0,0.5]$. It is clear that $F(\x)$ can be simultaneously minimized in $\x_\star=\left( 0,\ 0 \right)$, hence, this is an AMOO setting. This example, as we demonstrate, illustrates an instance in which each individual function \textit{does not specify the solution well}, but with proper weighting the optimal solution is well specified.

First, observe both $f_1$ and $f_2$ are $\Delta$-strongly convex and $O(1)$-smooth functions. Hence, GD with properly tuned learning rate, applied to either $f_1$ or $f_2$ will converge with linear rate of $\Omega(\Delta)$. It is simple to observe this rate can be  dramatically improved by proper weighting of the functions. Indeed, let $f_{\w_U}$ be a function with equal weighting of both $f_1$ and $f_2$, namely, choosing $\w_U=(0.5,0.5)$, we get $f_{\w_U}(\x)=0.5 x_1^2 + 0.5 x_2^2$
which is $\Omega(1)$-strongly convex and $O(1)$-Lipchitz smooth. Hence, GD applied to $f_{\w_U}$ converges with linear rate of $\Omega(1)$---much faster than $O(\Delta)$ when $\Delta$ is taken to be arbitrarily small.

\paragraph{The Selection Example.}
Consider the case $F(\x)=(f_1(\x),\ldots,f_m(\x))$, $\x\in \reals^m$, where
\begin{align*}
    & \forall i \in [m-1] : f_i(\x) = (1-\Delta)x_1^2 + \Delta \sum_{j=2}^{d} x_j , \quad \text{and} \quad  f_m(\x) = \sum_{j=1}^{d} x_j^2,
\end{align*}
and $\Delta\in [0,0.5]$. The common minimizer of all functions is $\x_\star=\bold{0} \in \reals^d$, and, hence, the objectives are aligned. Unlike the specification example, in the selection example, there is a single objective function among the $m$ objectives we should select to improve the convergence rate of GD. Further, in the selection example, choosing the uniform weight degrades the convergence rate. 

\hspace{-0.1cm}
\begin{minipage}{0.38\textwidth}
\begin{algorithm2e}[H]
 \caption{\texttt{AMOOO-GD}}\label{alg:AMOOO-GD}    
 \SetAlgoLined
  \While{$t=1,2,\ldots$}{
  { \color{orange} $\w_t \gets \texttt{AMOOO}\brac{\{ f_i(\x_t) \}_{i=1}^m}$ }\\
  $\g_t \gets \nabla f_{\w_t}(\x_t)$ \\
  $\x_{t+1} = \x_t - \eta_t \g_t$ \\
  }
\end{algorithm2e}
\end{minipage}
\hspace{0.2cm}
\begin{minipage}{0.55\textwidth}
\vspace{0.0cm}
\begin{algorithm2e}[H]
 \caption{\texttt{AMOOO}}\label{alg:AMOOO}    
 \SetAlgoLined
  \textbf{inputs:} $\{ f_i(\x_t) \}_{i=1}^m$ \\
  \textbf{initialize:} $\wmin=\mu_\star/\brac{8m\beta}$ \\
  Get Hessian matrices $\{ \nabla^2 f_i(\x_t) \}_{i=1}^m$ \\
  $\w_t\in \arg\max_{\w\in \Delta_{m,\wmin}} \lambda_{\min} \brac{\sum_{i} w_i \nabla^2 f_i(\x_{t})}$\\
  Return $\w_t$
\end{algorithm2e}
\end{minipage}

\vspace{0.4cm}

Indeed, setting the weight vector to be uniform $\w_{U}=\brac{1/m,\ldots,1/m}\in\reals^m$ leads to the function
$
    f_{\w_{U}}(\x) = (2-\Delta)/m \cdot x_1^2 +\sum_{j=2}^{d} (\Delta+1)/m \cdot x_j^2,
$
which is $O(1/m)$-strongly convex. Hence, GD applied to $f_{\w_{U}}$ converges in a linear rate of $O(1/m)$. On the other hand, GD applied to $f_m$ converges with linear rate of $\Omega(1)$. Namely, setting the weight vector to be $(0,\ldots,0,1)\in \reals^m$ improves upon taking the average when the number of objectives is large.

\section{Optimal Adaptive Strong Convexity \& The \texttt{AMOOO} Algorithm}

The aforementioned instances highlighted that in the AMOO setting the weights should be chosen in an adaptive way to the problem instance, and, specifically, based on the curvature. We formalize this intuition and design the AMOO-Optimizer (\texttt{AMOOO}). Towards developing it, we define the optimal adaptive strong convexity parameter, $\mu_\star$. Later we show that when the weighted loss is determined by \texttt{AMOOO} GD converges in a rate that depends on $\mu_\star$.

We start by defining the optimal adaptive strong convexity over the class of weights:
\begin{definition}[Optimal Adaptive Strong Convexity $\mu_\star$] \label{def:mu_star}
     The optimal adaptive strong convexity parameter, ${\mu_\star\in \reals_{+}}$, is the largest value such that $\forall \x\in \mX$ exists a weight vector $\w\in \Delta_m$ such that  
    \begin{align}
       \lambda_{\min}\left(\sum_{i=1}^m w_i \nabla^2 f_i(\x) \right)\geq \mu_\star. \label{eq:mu_star_definition}
    \end{align}
\end{definition}

For each $\x\in \mX$, there may be a different weight vector in class $w_\star(\x)\in \Delta_m$ that solves $w_\star(\x)\in \arg\max\lambda_{\min}\left(\nabla^2 f_\w(\x) \right)$ and maximizes the curvature. The optimal adaptive strong convexity parameter $\mu_\star$ is the largest lower bound on this quantity on the entire space $\mX.$ The specification and selection examples (Section~\ref{sec:AMOO setting}) demonstrate $\mu_\star$ can be much larger than both the strong convexity parameter of the average function or of each individual function; for both $\mu_\star=O(1)$ whereas the alternatives may have arbitrarily small strongly convex parameter value.

Definition~\ref{def:mu_star} not only quantifies an optimal notion of curvature, but also directly results with the \texttt{AMOOO} algorithm. \texttt{AMOOO} sets the weights according to equation~\ref{eq:mu_star_definition}, namely, at the $k^{\mathrm{th}}$ iteration, it finds the weight for which $f_\w(\x_k)$ has the largest local curvature. Then, a gradient step is applied in the direction of $\nabla f_{\w_k}(\x_k)$ (see Algorithm~\ref{alg:AMOOO-GD}). Indeed, \texttt{AMOOO} seems as a natural candidate for AMOO. One may additionally hope that standard GD analysis techniques for strongly-convex and smooth functions can be applied. It is well known that if a function $f(\x)$ is $\beta$ smooth and $\forall \x\in \mX,\ \lambda_{\min}\brac{\nabla^2 f(\x)} \geq \mu$ then GD converges with $\mu/\beta$ linear rate. 

A careful examination of this argument shows it fails. Even though $\lambda_{\min}\brac{\nabla^2 f_{\w_k}(\x_k)}\geq \mu_\star$ at each iteration it does not imply that $f_{\w_k}$ is $\mu_\star$ strongly convex for a fixed $\w_k$. Namely, it does not necessarily hold that for all $\x\in \mX,\ \lambda_{\min} \brac{\nabla^2 f_{\w_k}(\x)} \geq \mu_\star$, but only pointwise at $\x_k$. This property emerges naturally in AMOO, yet such nuance is inherently impossible in single-objective optimization and, to the best of our knowledge, was not explored in online optimization as well.

Next, we provide a positive result. When restricting the class of functions to the set of self-concordant and smooth functions (see Appendix~\ref{app:theory}) we provide a convergence guarantee for \texttt{AMOOO-GD} that depends on~$\mu_\star$. The result shows that close to the optimal solution \texttt{AMOOO-GD} converges linearly with rate of $O(\mu_\star/\beta)$.

\begin{restatable}[$\mu_\star$ Convergence of \texttt{AMOOO-GD}]{theorem}{ExactAmooConvergence}
\label{thm:exact_amoo_convergence}
    Suppose $\{f_i\}_{i\in [m]}$ are $\beta$ smooth, $M_f$ self-concordant, share an optimal solution $\x_\star$ and that $\mu_\star > 0$. Let  
     $k_0 := \lceil \brac{ \enorm{\x_0 - \x_\star} 3M_f\sqrt{m}\beta^{2}-\beta}/\mu_\star^{3/2} \rceil,$
    where $\enorm{\cdot}$ is the 2-norm. If $\eta_t = 1/2\beta$ then \texttt{AMOOO-GD} converges with rate:
    \begin{align*}
        \enorm{\x_t-\x_\star} \leq 
        \begin{cases}
            \brac{1-\mu_\star/8\beta}^{(k-k_0)/2} \sqrt{\mu_\star}/ 3M_f\sqrt{m}\beta & k\geq k_0\\
            \enorm{\x_0-\x_\star} - k \mu_\star^{3/2}/24M_f\sqrt{m}\beta & o.w.
        \end{cases}
    \end{align*}
\end{restatable}
Interestingly, Theorem~\ref{thm:exact_amoo_convergence} holds without making strong convexity assumption on the individual functions, but only requires that the adaptive strong convexity parameter $\mu_\star$ to be positive, as, otherwise, the result is vacuous.

\subsection{Practical Implementation}\label{sec:prac_imp}

Towards large scale application of  \texttt{AMOOO} with modern deep learning architectures we simplify its implementation. First, we optimize over the simplex as oppose to over $\Delta_{m,\min}$. We conjecture this is a by product of our analysis. In addition, we approximate the Hessian matrices with their diagonal.  Prior works used the diagonal Hessian approximation as pre-conditioner~\cite{chapelle2011improved,schaul2013no,yao2021adahessian,liu2023sophia,achituve2024bayesian}. Notably, with this approximation the computational cost of \texttt{AMOOO} scales linearly with number of parameters in the Hessian calculation, instead of quadratically. The following result establishes that the value of optimal curvature, and, hence the convergence rate of \texttt{AMOOO-GD}, degrades continuously with the quality of approximation.
\begin{restatable}{proposition}{ApproxHessian}
\label{thm:app_hessian}
    Assume that for all $i\in [m]$ and $\x\in \mX$ $|| \nabla^2 f_i(\x) - \mathrm{Diag}\brac{\nabla^2 f_i(\x)} ||_2\leq \Delta$ where $\enorm{\A}_2$ is the spectral norm of $\A\in \reals^{d\times d}$. Let $\widehat{\w} \in \arg\max_{\w\in \Delta_{m}} \lambda_{\min} \brac{\sum_{i} w_i \nabla^2 \mathrm{Diag}\brac{f_i(\x)}}$. Then,
    $
        \lambda_{\min} \brac{\sum_{i} \widehat{w}_i \nabla^2  f_i(\x) }\geq \mu_\star - 2\Delta.
    $
\end{restatable}

Next we provide high-level details of our implementation (also see Appendix~\ref{app:practical_implementation}).
\paragraph{Diagonal Hessian estimation via Hutchinson's Method.} We use the Hutchinson method~\cite{hutchinson1989stochastic, chapelle2011improved, yao2021adahessian} which provides an estimate to the diagonal Hessian by averaging products of the Hessian with random vectors. Importantly, the computational cost of the Hutchinson method scales linearly with number of parameters.

\paragraph{Maximizing the minimal eigenvalue.} Maximizing the minimal eigenvalue of symmetric matrices is known to be a convex problem (\citet{boyd2004convex}, Example~3.10) and can be solved via semidefinite programming. For diagonal matrices the problem can be cast as a simpler max-min bilinear problem, and, specifically, as
$
    \arg\max_{\w\in \Delta^m} \min_{\q\in \Delta^d} \w^T \A \q,
$
where $d$ is the dimension of parameters, $\A\in \reals^{m\times d}$ and its $i^{th}$ row is the diagonal Hessian of the $i^{th}$ objective, namely, $\forall i\in [m],\ \A[i,:]=\diag(\nabla^2 f_i(\x))$.

This bilinear optimization problem has been well studied in the past~\cite{rakhlin2013optimization,mertikopoulos2018mirror,daskalakis2018last}. We implemented the PU method of~\citet{cen2021fast} which, loosely speaking, performs iterative updates via exponential gradient descent/ascent. Note that, PU has a closed form update ruke and its computational cost scales linearly with number of parameters.

\section{Experiment}

We will compare our implementation of \texttt{AMOOO} to a weighting mechanism that equally weighting the objectives (\texttt{EWO}). Specifically,  we choose 10 axis-aligned quadratic losses of the form 
\begin{align}
    \small
    f_i(\x) = (h_\theta(\x) - h_{\theta_\star}(\x))^\top \H_i (h_\theta(\x) - h_{\theta_\star}(\x)), \quad \forall i \in [10],
\end{align}
where $\H_i\in \mathbb{R}^{10\times 10}$  is a diagonal positive semi-definite Hessian matrix. Both $h_{\theta_\star}:\reals^d \rightarrow \reals^d$ and $h_{\theta}:\reals^d \rightarrow \reals^d$ are 2-layer neural networks with parameters $\theta_\star$ and $\theta$. Observe that all of the loss functions are minimized when $h_\theta(\x)=h_{\theta_\star}(\x)$, and, hence, it is an instance of the AMOO setting.

In our experiment, we choose all but one of the losses to have low curvature, \textbf{simulating a selection example} (see Section~\ref{sec:AMOO setting}). The features $\x$ are generated by sampling from a $d$ dimensional Normal distribution $\mathcal{N}(0,\I_{10})$, and the targets are perturbed by an additional Normal noise, namely, $\y = h_{\theta_\star}(\x) +\epsilon_\sigma$ where $\epsilon_\sigma \sim \mathcal{N}(0,\sigma^2\I_{10})$, where $\I_{d}$ is the identity matrix in dimension $d$. We experiment with three different noise levels by modifying $\sigma$.
We test both \texttt{AMOOO} and \texttt{EWO} as the mechanisms for calculating a weighted loss $f_{\w}$ at each iteration, and apply either SGD or Adam optimizer to $f_\w$. In both cases we perform a grid search on the learning rate to find the best performing learning rate parameter.
In Figure \ref{fig:axis_aligned_selection},  we show the results of our simulation. Generally, \texttt{AMOOO} performs better than \texttt{EWO} in all settings across optimizers and noise levels. Adam (right) approaches a more optimal representation than SGD. See additional details in Appendix \ref{app:practical_implementation}.

\section{Conclusion}

In this work, we introduced the AMOO framework to study how aligned multi-objective feedback can improve gradient descent convergence. We designed the \texttt{AMOOO} algorithm, which adaptively weights objectives and offers provably improved convergence guarantees. Our experimental results demonstrate \texttt{AMOOO}'s effectiveness optimizing a large number of tasks that share optimal solution. Future research directions include determining optimal rates for AMOO and conducting comprehensive empirical studies. Such advancements will improve our ability to scale learning algorithms to handle a large number of related tasks efficiently.

\begin{figure*}[t]
\vspace{-8pt}
    \centering
    \begin{minipage}{0.45\textwidth}
        \centering
        \includegraphics[width=\textwidth]{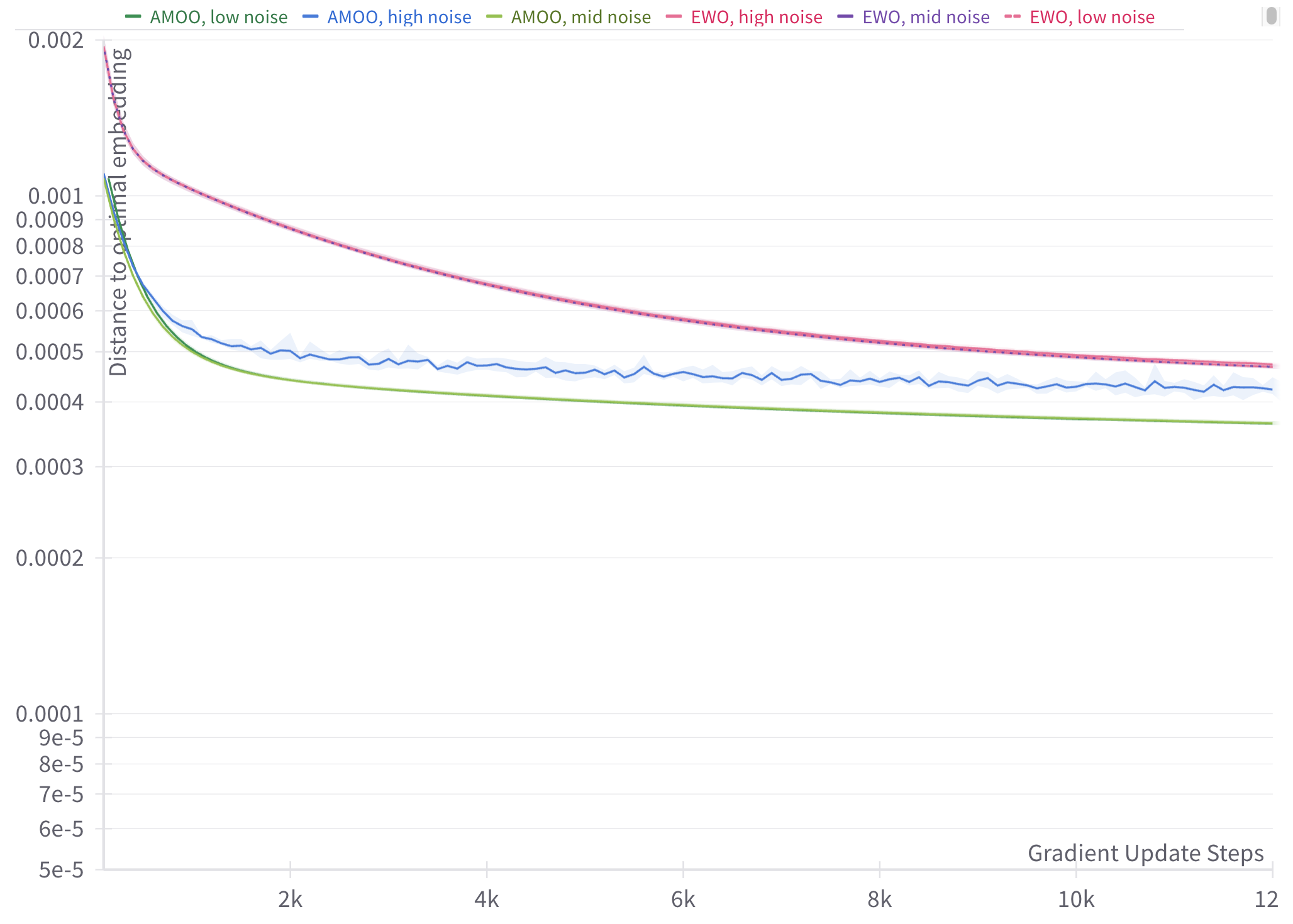} 
    \end{minipage}\hfill
    \begin{minipage}{0.45\textwidth}
        \centering
        \includegraphics[width=\textwidth]{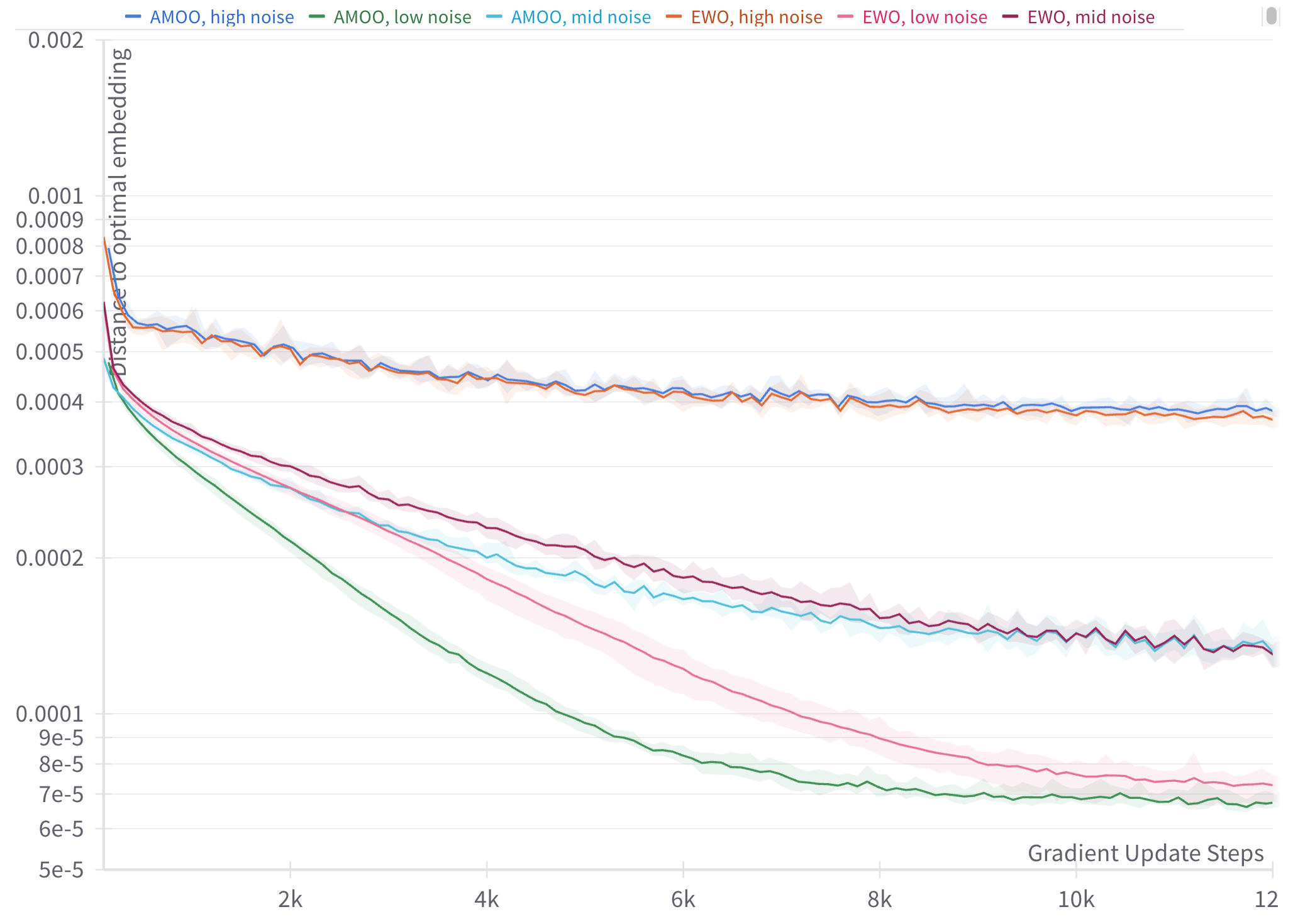} 
    \end{minipage}
    \caption{\texttt{AMOO} tested against equal weighting of loss functions (\texttt{EWO}) when optimized by SGD \textbf{(left)} or Adam \textbf{(right)}. Additionally, we test the effect of additive Normal noise of the optimal representation $h_\theta(x)$. \texttt{AMOO} performs better than its counterpart in all 6 settings. }
    \label{fig:axis_aligned_selection}
\end{figure*}




\bibliography{citation}
\appendix

\section{Related Work}
\subsection{Multi-task Learning and Gradient Weights}
Our work is closely related optimization methods from the multi-task learning (MTL) literature, particularly those that integrate weights into the task gradients or losses. The \emph{multiple gradient descent algorithm} (MGDA) approach of \cite{desideri2012multiple}, which proposes an optimization objective that gives rise to a weight vector that implies a descent direction for all tasks. MGDA converges to a point on the Pareto set. MGDA was introduced into the deep MTL setting in \cite{sener2018multi}, which propose extensions to MGDA weight calculation that can be more efficiently solved. 

The PCGrad paper \cite{yu2020gradient} identified that conflicting gradients, especially when there is high positive curvature and differing gradient magnitudes, can be detrimental to MTL. The authors then propose to alter the gradients to remove this conflict (by projecting each task's gradient to the normal plane of another task), forming the basis for the PCGrad algorithm. Another work that tackles conflicting gradients is the \emph{conflict-averse gradient descent} (CAGrad) method of \cite{liu2021conflict}. CAGrad generalizes MGDA: its main idea is to minimize a notion of ``conflict'' between gradients from different tasks, while staying nearby the gradient of the average loss. Notably, CAGrad maintains convergence toward a minimum of the average loss. Another way to handle gradient conflicts is the Nash-MTL method of \cite{navon2022multi}, where the gradients are combined using a bargaining game. Other optimization techniques for MTL include tuning gradient magnitudes so that all tasks train at a similar rate \citep{chen2018gradnorm}, taking the geometric mean of task losses \citep{chennupati2019multinet++}, and random weighting \citep{lin2021reasonable}.

Our approach, \texttt{AMOOO}, is similar to existing work in that it also computes gradient weights in order to combine information from multiple pieces of feedback. However, unlike previous work, we focus on exploiting prior knowledge that the objectives are \emph{aligned} and show both theoretically and empirically that such knowledge can be beneficial for optimization.

\subsection{Proxy, Multi-fidelity, and Sim-to-real Optimization}
Two other streams of related work are (1) machine learning using proxies and (2) multi-fidelity optimization. These works stand out from MTL in that they both focus on using \emph{closely related} objectives, while traditional MTL typically considers a set of tasks that are more varied in nature. Proxy-based machine learning attempts to approximate the solution of a primary ``gold'' task (for which data is expensive or sparsely available) by making use of a proxy task where data is more abundant \citep{bastani2021predicting,dzyabura2019accounting}. 

Similarly, multi-fidelity optimization makes use of data sources of varying levels of accuracy (and potentially lower computational cost) to optimize a target objective \citep{forrester2007multi}. In particular, the idea of using multiple closely-related tasks of varying levels of fidelity has seen adoption in settings where function evaluations are expensive, including bandits \citep{kandasamy2016multi,kandasamy2016gaussian}, Bayesian optimization \citep{kandasamy2017multi,song2019general,wu2020practical,takeno2020multi}, and active learning \citep{yi2021active,li2020deep,li2022batch}. Sim-to-real learning can be thought of as a particular instance of multi-fidelity optimization, where one hopes to learn real world behavior via simulations (typically in robotics) \citep{peng2018sim, zhao2020sim}. In many of these papers, however, the objectives are queried one at a time, differing slightly from the MTL or AMOO settings.

The motivations behind the AMOO setting are clearly similar to those of proxy optimization, multi-fidelity optimization, and sim-to-real learning. However, our papers takes a pure optimization and gradient-descent perspective, which to our knowledge, is novel in the literature.



\section{Proofs of Theoretical Results} \label{app:theory}

\subsection{Assumptions \& Consequences}

In this section we formally provide our working assumptions. We assume access to multi-objective feedback with $m$ objectives $F(\x) = (f_1(\x),\ldots,f_m(\x))$. Considering AMOO, we assume the functions are aligned in the sense of equation~\ref{eq:aligned_functions}, namely, that they share an optimal solution. 

We assume that the exist a local weighting for which the the minimal eigenvalue of the Hessian of the weighted function is at least $\mu_\star$. Further, we define the following quantities, for the single and multi optimization settings:
\begin{align*}
    &\enorm{\y}^2_{\x}:=\enorm{\y}_{\nabla^2 f(\x)}\\
    &\enorm{\y}^2_{\x,\w} :=\enorm{\y}_{\nabla^2 f_\w(\x)}.
\end{align*}


\begin{assumption}[Smoothness] \label{ass:smooth}
    All function are $\beta$-smooth. $\forall i\in [m]$, $f_i: \reals^n \xrightarrow{} \reals$ it holds that $~\forall \x, \y \in \mX$:
    \begin{align}
        f(\y) \leq  f(\x) + \nabla f(\x)^{\top} (\y - \x) + \frac{\beta}{2}  \enorm{ \x - \y } ^2. \nonumber
    \end{align}
\end{assumption}
\begin{assumption}[Self-concordant] \label{ass:self_con}
    All functions are self-concordant with $M_f \geq 0$ parameter.  $\forall i\in [m]$  $f: \reals^n \xrightarrow{} \reals$ and $~\forall \x,\y\in \mX$:
    \begin{align}
        \inner{\nabla^3 f_i(\x)[\y]\y}{\y}  \preceq 2M_f \enorm{\y}_{\x}^3, \nonumber
    \end{align}
    where $\nabla^3 g(\x)[\y] := \lim_{\alpha\rightarrow 0} \frac{1}{\alpha}\brac{\frac{\nabla^2 g(\x+\alpha \y) - \nabla^2 g(\x)}{\alpha}}$ is the directional derivative of the hessian in~$\y$.
\end{assumption}

These assumptions have  the following important consequences.
\begin{lemma}[Theorem 5.1.8 \& Lemma 5.1.5,~\citet{nesterov2013introductory}]\label{lemma:self_con_consequences}
Let $f:\mX \rightarrow \reals$ be an $M_f$ self-concordant function. Let $x, y\in\mX$ , we have
$$
f(\y) \geq f(\x) + \inner{\nabla f(\x)}{\y  - \x} + \frac{1}{M_f^2} \omega\brac{M_f\enorm{\y-\x}_\x },
$$
where $\omega(t) = t-\ln(1-t)$, and, for any $t>0$, $\omega(t)\geq \frac{t^2}{2(1+t)}$.
\end{lemma}
\begin{lemma}[Theorem 5.1.1, ~\citet{nesterov2018lectures}]\label{lemma:sum_of_self_con}
    Let $f_1,f_2 : \mX \to \reals$ be $M_f$ self-concordant functions. Let $w_1,w_2 > 0$. Then, $f=w_1 f_1+ w_2 f_2$ is $M = \max_i \{ \frac{1}{\sqrt{w_i}} \} M_f$ self-concordant function.
\end{lemma}
\begin{restatable}{lemma} {SumSelfCon} \label{lemma:sun_of_self_con}
    Let $\{ f_i: \mX \to \reals \}_{i=1}^{n} $ be $M_f$ self-concordant functions. Let $\{w_i > 0\}$. Then, $f= \sum_{i=1}^{n} w_i f_i$ is $M = \max_i \{ \frac{1}{\sqrt{w_i}} \} M_f$ self-concordant function.
\end{restatable}
\begin{proof}
    Let $f= \sum_{i=1}^{n} w_i f_i$. We prove it by using induction. \\
    \emph{Basis:} $n=2$. Using Lemma \ref{lemma:sum_of_self_con} we obtain that $f$ is $\max_{i\in[1,2]} \{ \frac{1}{\sqrt{w_i}} \} M_f$ self-concordant function.\\
    \emph{Induction assumption:} For every $n<k$ it hold that $f$ is $\max_{i\in[1,n]} \{ \frac{1}{\sqrt{w_i}} \} M_f$ self-concordant function.\\
    \emph{Induction step:} Let $f= \sum_{i=1}^{k} w_i f_i$. Define $g = \sum_{i=1}^{k-1} w_i f_i $. From the \emph{Induction assumption} it hold that $g$ is $\max_{i\in[1,k-1]} \{ \frac{1}{\sqrt{w_i}} \} M_f$ self-concordant function. Since $f = g + w_k f_k$, by using Lemma~\ref{lemma:sum_of_self_con} we obtain that $f$ is $\max \{ \max_{i\in[1,k-1]} \{ \frac{1}{\sqrt{w_i}} \}, \frac{1}{\sqrt{w_k}} \} M_f = \max_{i\in[1,k]} \{ \frac{1}{\sqrt{w_i}} \} M_f$ self-concordant function.
\end{proof}
\begin{lemma}[Standard, E.g., 9.17~\citet{boyd2004convex}] \label{lemma:smooth-gradient-norm}
    Let $f : \reals^n \to \reals$ a $\beta$-smooth over $\mX$, and let $\x_\star \in \underset{\x \in \reals}{\arg\min} ~ f(\x)$.  Then, it holds that
    \begin{align*}
        \enorm{\nabla f (\x)}^2 \leq 2\beta \brac{f(\x) - f(\x_\star)}.
    \end{align*}
\end{lemma}
Further, we have the following simple consequence of the AMOO setting.
\begin{lemma}\label{lemma:optimality_of_x_star}
For all $\w\in \Delta_m$ and $\x\in \mX$ it holds that $f_\w(\x) - f_\w(\x_\star)\geq 0.$
\end{lemma}
\begin{proof}

    Observe that
    $
        f_\w(\x) - f_\w(\x_\star)= \sum_{i=1}^m w_i \brac{f_i(\x) - f_i(\x_\star)}.
    $
    Since $\x_\star$ is the optimal solution for all objectives it holds that $f_i(\x) - f_i(\x_\star) \geq  0.$ The lemma follows from the fact $w_i\geq 0$ for all $i\in [m].$
\end{proof}

Further, recall that the following observation holds.
\begin{observation}\label{obs:simplex_sum}
    Let $\w \in \Delta_m$. Assume $\{ f_i \}_{i=1}^m$ are  $ \beta $ smooth. Then, $f_{\w}(\x) := \sum_{i=1}^{m} w_i f_i(\x)$ is also $\beta$ smooth.
\end{observation} 

\subsection{Proof of Proposition~\ref{thm:app_hessian}}

Recall the following results which is a corollary of Weyl's Theorem.
\begin{theorem}[Weyl's Theorem]\label{thm:weyls} Let $\A$ and $\Delta$ be symmetric matrices in $\reals^{d\times d}$. Let $\lambda_j(\A)$ be the jth largest eigenvalue of a matrix $\A$. Then, for all $j\in [d]$ it holds that $\| \lambda_j(\A) - \lambda_j(\A+\Delta) \| \leq \| \Delta\|_2$, where $\| \Delta\|_2$ is the operator norm of $\Delta$.
\end{theorem}
Proposition~\ref{thm:app_hessian} is a direct outcome of this result. We establish it for a general deviation in Hessian matrices without requiring it to be necessarily diagonal.

\begin{proof}
    
Denote $\A_i := \nabla^2 f(\x) + \Delta_i$. Let $\w_\star $ denote the solution of,
    \begin{align*}
        \w_\star \in \arg\max_{\w\in \Delta} \lambda_{\min}\brac{\sum_i w_i \nabla^2 f_i(\x)},
    \end{align*}
    and let $g(\w_\star)$ denote the optimal value, $g(\w_\star) = \lambda_{\min} \brac{\sum_i w_{\star, i}  \nabla^2 f_i(\x)}$. Similarly, let $\hat{\w}_\star$ denote the solution of the optimization problem of the perturbed problem:
    \begin{align*}
        \hat{\w}_\star \in \arg\max_{\w\in \Delta} \lambda_{\min}\brac{\sum_i w_i \A_i},
    \end{align*}
    and denote $\hat{g}(\hat{\w}_\star)$ as its value, $\hat{g}(\hat{\w}_\star) = \lambda_{\min} \brac{\sum_i \hat{w}_{\star, i} \A_i}$. Then, the following holds.
    \begin{align*}
        &g(\w_\star) = g(\w_\star) - \hat{g}(\w_\star) + \hat{g}(\w_\star) -\hat{g}(\hat{\w}_\star) +\hat{g}(\hat{\w}_\star) - g(\hat{\w}_\star) +g(\hat{\w}_\star) \\
        &\stackrel{(1)}{\leq}  g(\w_\star) - \hat{g}(\w_\star)  +\hat{g}(\hat{\w}_\star) - g(\hat{\w}_\star) +g(\hat{\w}_\star) \\
        &\leq |g(\w_\star) - \hat{g}(\w_\star)|  + |\hat{g}(\hat{\w}_\star) - g(\hat{\w}_\star)| +g(\hat{\w}_\star)\\
        & \stackrel{(2)}{\leq}  2 \| \Delta \|_2 +g(\hat{\w}_\star).
    \end{align*}
    $(1)$ holds since $\hat{g}(\w_\star) -\hat{g}(\hat{\w}_\star)\leq 0$ by the optimality of $\hat{\w}_\star$ on $\hat{g}$. Further, $(2)$ holds due to Weyl's Thoerem (Theorem~\ref{thm:weyls}) and the assumptions of the approximation error. Recall that for any $\w\in \Delta_m$ it holds that 
    $$
    \enorm{ \sum_i w_i \A_i -  \sum_i w_i \nabla^2 f_i(\x) }_2 \leq \sum_i w_i \enorm{ \A_i -  \nabla^2 f_i(\x) }_2 \leq \enorm{ \Delta }_2
    $$
    since $\sum_i w_i =1$. Hence, by Weyl's theorem it holds that 
    \begin{align*}
        |g(\w_\star) - \hat{g}(\w_\star)\| \leq \| \Delta \|_2 \text{ and } |g(\hat{\w}_\star) - \hat{g}(\hat{\w}_\star)\| \leq \| \Delta \|_2.
    \end{align*}
    Finally, since $g(\w_\star)\geq \mu_\star$, by Definition~\ref{def:mu_star}, we get that 
    \begin{align*}
        g(\hat{\w}_\star) \geq \mu_\star -2\| \Delta\|_2,
    \end{align*}
    which concludes the proof.

\end{proof}

\subsection{Proof of Theorem~\ref{thm:exact_amoo_convergence}}

In highlevel, the proof follows the standard convergence analysis of $\mu$-strongly convex and $L$-smooth function, while applying Lemma~\ref{lemma:self_con_consequences}, instead of using only properties of strongly convex functions alone.

In addition, we choose the minimal weight value, $\wmin$, such that the weighted function at each iteration $f_{\w_k}$ has a sufficiently large self-concordant parameter, while the minimal eigenvalue of its Hessian is close to $\mu_\star$. Before proving Theorem~\ref{thm:exact_amoo_convergence}, we provide two results that allow us to control these two aspects.
\begin{lemma}\label{lemma:self_concordant}
For any iteration $k$, the function $f_{\w_k}$ is $1/\sqrt{\wmin}M_f$ self-concordant.    
\end{lemma}
\begin{proof}
    This is a direct consequence of Lemma~\ref{lemma:sun_of_self_con} and the fact Algorithm~\ref{alg:AMOOO} sets the weights by optimizing over a set where the weight vector, $\w$. is lower bounded by $\wmin$.
\end{proof}
\begin{lemma}\label{lemma:mu_star_minimal_eigenvalue}
For any iteration $k$, we have $\lambda_{\min}\brac{\nabla^2f_{\w_k}}\geq \mu_\star - 2m\wmin\beta$.
\end{lemma}
\begin{proof}

    To establish the lemma we want to show that for any $\w\in \Delta_{m}$ there exists $\widehat{\w}\in \Delta_{m,\wmin}$ such that $\lambda_{\min} \brac{\sum_{i} \hat{w}_i \nabla^2 f_i(\x_{t})}\geq \lambda_{\min} \brac{\sum_{i} w_i \nabla^2 f_i(\x_{t})}-\wmin\beta$. We start by bounding the following term $\enorm{\nabla^2 f_\w(\x) - \nabla^2 f_{\hat{\w}}(\x)}_2$ for any $\x\in \mX$. We have 
    \begin{align*}
       &\enorm{\sum_i(w_i-\hat{w}_i)\nabla^2 f_i(\x)}_2 \leq  \sum_i |w_i-\hat{w_i}|\enorm{\nabla^2 f_i(\x)}_2  \leq \beta \sum_i |w_i-\hat{w_i}|,
    \end{align*}
    while the last inequality holds since $\{f_i\}_{i\in[m]}$ are $\beta$ smooth. Since for any $\w\in \Delta_m$ there exist $\hat{\w}\in \Delta_{m,\wmin}$ such that $\sum_i |w_i-\hat{w_i}| \leq 2m\wmin$, we obtain that for every $\x\in\mX$ it holds that 
    \begin{align*}
       &\enorm{\nabla^2 f_\w(\x) - \nabla^2 f_{\hat{\w}}(\x)}_2 \leq  2m\wmin\beta.
    \end{align*}
    Thus, by using Theorem~\ref{thm:weyls} we have 
    \begin{align*}
         | \lambda_{\min} (\nabla^2 f_\w(\x)) - \lambda_{\min} (\nabla^2 f_{\hat{\w}}(\x)) | \leq \enorm{\nabla^2 f_\w(\x) - \nabla^2 f_{\hat{\w}}(\x)}_2 \leq  2m\wmin\beta.
    \end{align*}
    Recall that $\lambda_{\min} (\nabla^2 f_\w(\x)) \geq \mu^*$ assuming Definition~\ref{def:mu_star} holds. Then, we obtain 
    \begin{align*}
         \lambda_{\min} (\nabla^2 f_{\hat{\w}}(\x))  \geq  \mu^* - 2m\wmin\beta.
    \end{align*}
    

\end{proof}
With these two results we are ready to prove Theorem~\ref{thm:exact_amoo_convergence}.

\begin{proof}
    
The GD update rule is given by $\x_{k+1} = \x_k - \eta\nabla f_{\w_k}(\x_k)$, where $\eta$ is the step size, and $\w_k\in\arg\max_{\w\in \Delta_{m}} \lambda_{\min} \brac{\sum_{i} w_i \nabla^2 f_i(\x_{t})}$. With the assumption that $\max_{\w\in \Delta_m}\lambda_{\min}\brac{\nabla^2 f_{\w_k}(\x_k)}=\mu_\star>0$, Lemma~\ref{lemma:mu_star_minimal_eigenvalue}, and since we set $\wmin=\mu_\star/\brac{8m\beta}$ we have that
\begin{align}
    \lambda_{\min}\brac{\nabla^2 f_{\w_k}(\x_k)}\geq \mu_\star-4m\wmin\beta:=\mu_\star/2=\widehat{\mu_\star}, \label{eq:implications_of_AMOOO}
\end{align}
for all iterations $k$.

We bound the squared distance between $\x_{k+1}$ and $\x_\star$:
\begin{align*}
    \enorm{\x_{k+1}-\x_\star}^2 &= \enorm{\x_k-\eta\nabla f_{\w_k}(\x_k)-\x_\star}^2 \\
    &=\enorm{\x_k-\x_\star}^2 - 2\eta\inner{\nabla f_{\w_k}(\x_k)}{\x_k - \x^*} + \eta^2\enorm{\nabla f_{\w_k}(\x_k)}^2
\end{align*}
By Lemma~\ref{lemma:self_concordant} it holds that $f_{\w_k}$ is 
$$\widehat{M_f}:=1/\sqrt{\wmin} M_f\leq 3\sqrt{m\beta}M_f/\sqrt{\mu_\star}$$ 
self concordant. Then, by applying Lemma~\ref{lemma:self_con_consequences} with $\y=\x_\star$ and $\x = \x_k$ we have
\begin{align*}
      \inner{\nabla f_{\w_k}(\x_k)}{\x_k-\x_\star } \geq f_{\w_k}(\x_k) -f_{\w_k}(\x_\star) + \frac{1}{\widehat{M_f}} \omega\brac{\widehat{M_f}\enorm{\x_\star - \x_k}_{\x,\w_k} }.
\end{align*}
which allows us to bound $ \enorm{\x_{k+1}-\x_\star}^2$ by 
\begin{align*}
    &\enorm{\x_k-\x_\star}^2 - 2\eta\brac{f_{\w_k}(\x_k) - f_{\w_k}(\x_\star) + \frac{1}{\widehat{M_f}} \omega\brac{\widehat{M_f}\enorm{\x_\star - \x_k}_{x,\w_k} }} + \eta^2\enorm{\nabla f_{\w_k}(\x_k)}^2\\
    &\stackrel{(1)}{\leq} \enorm{\x_k-\x_\star}^2 - \frac{2\eta}{\widehat{M_f}^2} \omega\brac{\widehat{M_f}\enorm{\x_\star - \x_k}_{\x,\w_k} } + 2\eta\brac{2\beta\eta-1} \brac{ f_{\w_k}(\x_k) -f_{\w_k}(\x_\star)}\\
    &\stackrel{(2)}{\leq} \enorm{\x_k-\x_\star}^2 - \frac{1}{\beta \widehat{M_f}^2} \omega\brac{\widehat{M_f}\enorm{\x_\star - \x_k}_{\x,\w_k} }\\
    &\stackrel{(3)}{\leq} \enorm{\x_k-\x_\star}^2 -  \frac{1}{2\beta}\frac{\enorm{\x_\star - \x_k}_{\x,\w_k}^2}{1+\widehat{M_f}\enorm{\x_\star - \x_k}_{\x,\w_k}}\\
    &\stackrel{(4)}{\leq} \enorm{\x_k-\x_\star}^2 -  \frac{\widehat{\mu_\star}}{2\beta}\frac{\enorm{\x_\star - \x_k}^2}{1+\widehat{M_f} \sqrt{\beta}\enorm{\x_\star - \x_k}}
\end{align*}
where $(1)$ is due to Lemma~\ref{lemma:smooth-gradient-norm}, $(2)$ holds by $f_{\w_k}(\x_k) -f_{\w_k}(\x_\star)\geq 0$ (Lemma~\ref{lemma:optimality_of_x_star}) and $\eta\brac{2\beta\eta-1}\leq 0$ (since $0<\eta\leq 1/2\beta$), $(3)$ is due to the lower bound on $\omega(t)$ from Lemma~\ref{lemma:self_con_consequences}, and $(4)$ follows from equation~\eqref{eq:implications_of_AMOOO} and since $f_\w$ is $\beta$ smooth for all $\w\in \Delta_m$.

 The above recursive equation results in polynomial contraction for large $\enorm{\x_\star - \x_k}$, and, then exhibits linear convergence. To see this, let $\kappa:=\frac{\widehat{\mu_\star}}{\beta}$, and examine the two limits.

 \paragraph{Linear convergence, $\enorm{\x_\star - \x_k}\leq \delta/\widehat{M_f}\sqrt{\beta},\ \delta\leq 1$.} With this assumption we have the following bound on the recursive equation:
 \begin{align*}
     \enorm{\x_{k+1}-\x_\star}^2 \leq \brac{1-\frac{\kappa}{2(1+\delta)}}\enorm{\x_{k}-\x_\star}^2.
 \end{align*}
 By setting $\delta=1$ we get the result. Further, $\enorm{\x_{k+1}-\x_\star}^2$ contracts monotonically, without exiting the ball $\enorm{\x_\star - \x_k}\leq \delta/\widehat{M_f}\sqrt{\beta}$, the linear convergence rate approaches $\kappa/2$. 

 \paragraph{Polynomial convergence, $\enorm{\x_\star - \x_k}> 1/\widehat{M_f}\sqrt{\beta}$.} With this assumption we have the following bound:
 \begin{align*}
     \enorm{\x_{k+1}-\x_\star}^2 \leq \enorm{\x_{k}-\x_\star}^2 - \frac{\kappa}{4\widehat{M_f} \sqrt{\beta}}\enorm{\x_{k}-\x_\star}.
 \end{align*}
This recursive equation decays in a linear rate and have the following closed form upper bound $\enorm{\x_{k+1}-\x_\star}\leq \enorm{\x_{0}-\x_\star} - k \frac{\kappa}{8\widehat{M_f}\sqrt{\beta}}$.

By plugging the values of $\widehat{M_f}$ and $\widehat{\mu}_\star$ we obtain the final result.
\end{proof}

\section{Practical Implementation}\label{app:practical_implementation}

\textbf{Dataset.}
We generate 10 dimensional inputs, $\x\in \reals^{10}$ from an independent Normal distribution $\mathcal{N}(0, \I_d)$. 
The target generating network $h_{\theta_\star}$ is randomly generated. The noise on targets is sampled from a Normal distribution $\epsilon_\sigma\sim \mathcal{N}(0, \sigma \I_d)$ and the noise level is either high $\sigma=1$, medium size  $\sigma=0.1$  or low $\sigma=0.001$.

\textbf{Network architecture.}
We choose the ground truth network and target network to have the same architecture. Both are 2 layer neural networks with 256 hidden dimensions and ReLu activation. The neural network outputs a vector in dimension $10$, similar to the input of the network.

\textbf{Loss functions.}
We choose $\H_1= \I_{10}$, and for $i>1$ $\H_{i}= \alpha \I_{10} + (1-\alpha)\A$ and $\alpha=10^{-4}$ where 
$$
\A_{i,j}=
\begin{cases}
    1 & i=j=1 \\
    0 & o.w.,
\end{cases}
$$ 
namely, $\A$ is a diagonal matrix with value of $1$ in the first diagonal index and zero otherwise. 

In this problem,  the function generated by the $\H_1$ Hessian has the largest minimal eigenvalue and we expect $\texttt{AMOOO}$ to choose this function, whereas \texttt{EWO} gives equal weight to every loss function.

\textbf{Training.}
We optimize learning rates across a grid of candidates and pick the best performing one on training loss $[1e-5, 1e-4,1e-3,1e-2]$, $1e-3$ performed best in all settings.
We choose a batch size of 1024. 
We perform each run with 5 different seeds and average their performance.

\textbf{General parameters for \texttt{AMOOO}.}
We set the number of samples for the Hutchinson method to be $N_{\mathrm{Hutch}}=100$. Namley, we estimate the Hessian matrices by averaging $N_{\mathrm{Hutch}}=100$ estimates obtained from the Hutchnison method.  Additionally, we use exponential averaging to update the Hessian matrices with $\beta=0.99.$ Further, at each training step we perform a single update of the weights based on the PU update rule of~\citet{cen2021fast} to solve the max-min Bilinear optimization problem (see Section~\ref{sec:prac_imp}).

\textbf{Validation.}
We measure the $L_2$ distance between $h_\theta$ and $h_{\theta_\star}$ averaged over $1024\cdot10^3$ validation points and measured per dimension. This quantity suppose to approximate the quality of the learned model $\theta$ which is given by $\mathbb{E}_{x}\left[\enorm{h_\theta(\x)-h_{\theta_\star}(\x)}^2\right]$.


\end{document}